\documentclass{article}

\PassOptionsToPackage{numbers, compress}{natbib}

\usepackage[preprint]{neurips_2026}
\usepackage{multirow}

\usepackage[utf8]{inputenc} 
\usepackage[T1]{fontenc}    
\usepackage{hyperref}       
\usepackage{url}            
\usepackage{booktabs}       
\usepackage{amsfonts}       
\usepackage{nicefrac}       
\usepackage{microtype}      
\usepackage{xcolor}         
\usepackage{tikz}           
\usepackage{wrapfig}
\usepackage{graphicx}
\usepackage{amsmath}
\newtheorem{proposition}{Proposition}
\usepackage{algorithm}
\usepackage{algpseudocode}

\usepackage{pifont}
\newcommand{\cmark}{\ding{51}}   
\newcommand{\xmark}{\ding{55}}   
\newcommand{\tmark}{\textasciitilde} 

\newcommand{\DKL}{D_{\mathrm{KL}}}
\newcommand{\DJS}{D_{\mathrm{JS}}}
\newcommand{\DJSjoint}{\DJS^{\mathrm{joint}}}

\title{Not all Jensen-Shannon Divergence Estimators are Equal}

\author{%
  Alba Garrido\textsuperscript{1} \quad 
  Alejandro Almodóvar\textsuperscript{1} \quad 
  Mar Elizo\textsuperscript{1} \\
  \textbf{Patricia A. Apellániz}\textsuperscript{1} \quad 
  \textbf{Santiago Zazo}\textsuperscript{1} \quad 
  \textbf{Juan Parras}\textsuperscript{1} \\
  \vspace{0.2cm}\\
  \textsuperscript{1}Information Processing and Telecommunications Center, ETSI Telecomunicación,  \\
  Universidad Politécnica de Madrid, Spain 
}

\begin{document}

\maketitle

\begin{abstract}
The Jensen-Shannon divergence is widely reported as a scalar measure of fidelity for synthetic tabular data. Yet, in practice, it is estimated from finite samples using protocols that are often underspecified. This creates a measurement problem. Although the population divergence is well defined, the empirical value depends on the estimator family, sampling protocol, calibration, dimensionality, and class balance. We show that different protocols can yield non-comparable values: marginal-based estimators ignore dependencies in the joint distribution and can severely underestimate divergence, while classifier-based estimators capture joint structure but exhibit strong estimator dependence. We systematically study this behavior across controlled settings with reference divergences and real-world synthetic tabular benchmarks. Our analysis reveals dependence blindness in marginal estimators, prior-shift bias under class imbalance, and estimator sensitivity in high dimensions. To address prior shift, we derive a closed-form posterior correction for classifier-based Jensen-Shannon estimation. Our results show that empirical Jensen-Shannon divergence values are inherently protocol-dependent, making explicit specification of the estimation procedure necessary for meaningful comparison. We provide practical guidelines and an open-source tool for estimator-aware Jensen-Shannon evaluation.
\end{abstract}

\section{Introduction}
\label{sec:intro}
Evaluating synthetic-data quality is a central challenge in machine learning, particularly for privacy-preserving generation and tabular tasks such as classification and risk prediction \cite{kaabachi2025scoping, pilgram2025consensus, achterberg2025fidelity}. Among validation metrics, the Jensen-Shannon divergence ($\DJS$) is widely used for its symmetry, boundedness, and information-theoretic foundation as a smoothed Kullback-Leibler divergence ($\DKL$) \cite{lin2002divergence, dorent2025connecting}. In tabular data, reported $\DJS$ values are often interpreted as quantitative evidence of distributional fidelity \cite{achterberg2025fidelity, espinosa2023quality, stenger2024evaluation}, yet this interpretation assumes that such values are well-defined and comparable across studies. Since no closed-form estimator exists in the general case, $\DJS$ must be estimated from finite samples \cite{aminian2021jensen, perez2008kullback, nguyen2010estimating} through often underspecified protocols. This raises a fundamental question: \emph{when two practitioners report a $\DJS$ estimate, are they measuring the same quantity?}

A common approach to estimating $\DJS$ computes feature-wise (marginal) divergences and averages them across dimensions \cite{xu2019ctgan, liu2024scaling}. Although efficient, this strategy ignores dependencies between variables and can underestimate joint discrepancies. However, matching marginal distributions does not guarantee preservation of the data-generating process. For instance, in clinical data, variables such as age and comorbidity are often strongly dependent, with older patients exhibiting a higher burden of comorbidities \cite{barnett2012epidemiology}. A synthetic dataset may faithfully reproduce the marginal distribution of each variable while failing to capture such relationships, yielding nearly identical marginals but substantially different multivariate structure, as illustrated in Figure \ref{fig:marginals}. Recent work confirms that marginal metrics are insufficient and must be complemented by joint-distribution measures \cite{riasat2026dependence, walia2020synthesising}.

An alternative estimates $\DJS$ on the joint distribution using probabilistic classifiers, leveraging the connection between optimal discriminators and density ratios \cite{sugiyama2012density}. Framing divergence estimation as binary classification captures full multivariate discrepancies without relying on marginal decompositions, making them suitable for detecting differences arising from dependencies between variables \cite{lopez2016C2ST}. However, this flexibility has a cost: the estimate is no longer solely a property of the distributions, but also of the estimator. Yet, to the best of our knowledge, there is no systematic analysis of how classifier-based $\DJS$ estimators behave across model classes and data regimes.

We address this gap with a systematic empirical study of classifier-based $\DJS$ estimation across classifier families. We evaluate estimators in controlled settings that isolate dependence structure, sample size, dimensionality, and class imbalance, and on real-world synthetic data from Variational Autoencoders (VAEs) \cite{kingma2013vae} and Generative Adversarial Networks (GANs) \cite{goodfellow2014generative}. Although we focus on tabular data, the measurement perspective we advocate extends naturally to other data modalities, such as images, where similar fidelity-assessment challenges arise \cite{kelkar2023assessing}.

Our central premise is simple: not all $\DJS$ estimates are equal, because not all estimators are equal. We show that different estimation protocols on the same data can yield substantially different $\DJS$ values, leading to inconsistent model-evaluation conclusions. Thus, empirical divergence values should be interpreted as protocol-dependent measurements, rather than properties of generated data. Our contributions are:
\begin{itemize}
    \item We formalize divergence estimation as a protocol-dependent measurement and argue that reported values require an explicit estimator specification (model, calibration, and protocol).
    \item We identify three failure modes (dependence structure, class imbalance, and dimensionality) and propose a closed-form posterior correction restoring consistency under prior shift.
    \item We show that marginal estimators severely underestimate divergence by ignoring joint dependencies, while classifier estimators vary substantially across model and data regimes.
    \item We characterize the accuracy-cost trade-off across estimators and release practical guidelines and an open-source implementation for estimator-aware $\DJS$ evaluation. \footnote{Available at \url{https://github.com/AlbaGarridoLopezz/jensenshannondivergence} 
and PyPI (\texttt{pip install jensenshannondivergence}).}
    
\end{itemize}
\section{Background}
\label{sec:background}

\subsection{Jensen-Shannon divergence estimation}
\label{subsec:js_divergence}
Let $P$ and $Q$ be probability distributions over the same set $\mathcal{X}$ with densities $p(x)$ and $q(x)$. The Jensen-Shannon divergence ($\DJS$) is defined as
\begin{equation}
\mathrm{D_{JS}}(P \parallel Q)
=
\frac{1}{2}\mathrm{D_{KL}}(P \parallel M)
+
\frac{1}{2}\mathrm{D_{KL}}(Q \parallel M);
\qquad
\mathrm{D_{KL}}(P \parallel Q)
=
\int_{\mathcal{X}} p(x)\log\frac{p(x)}{q(x)}\,dx
\end{equation}
where $M=\frac{1}{2}(P+Q)$ is the mixture distribution, $\DKL$ denotes the Kullback-Leibler divergence and $\log$ is taken in base 2. Under this convention, $\DJS$ is symmetric and bounded in $[0,1]$. Despite these appealing properties, computing $\mathrm{D_{JS}}(P \parallel Q)$ is generally intractable in practice: $M$ rarely admits a closed-form expression, even for simple parametric families \cite{nielsen2019jensen}, and real-world distributions are typically accessible only through finite samples \cite{goodfellow2014generative}. Thus, $\DJS$ must be estimated from data. 

\subsection{Marginal estimations do not recover the joint divergence}
A common strategy in synthetic data evaluation is to approximate $\DJS$ by averaging marginal divergences across dimensions. Given a $d$-dimensional random vector $X = (X_1, \dots, X_d)$,
\begin{equation}
D_{\mathrm{JS}}^{\mathrm{marg}}(P \parallel Q)
=
\frac{1}{d} \sum_{i=1}^{d}
\mathrm{D_{JS}}\big(P(X_i) \parallel Q(X_i)\big).
\end{equation}
where $P(X_i)$ and $Q(X_i)$ denote the marginal distributions of the $i$-th component. While computationally efficient, this estimator fundamentally ignores dependencies between variables, effectively treating each dimension independently. 

\begin{proposition}
There exist distributions $P$ and $Q$ over $\mathbb{R}^d$ with identical one-dimensional marginals, i.e., $P(X_i)=Q(X_i)$ for all $i=1$, but $P\neq Q$; hence $D^{\mathrm{marg}}_{\mathrm{JS}}(P \parallel Q)=0$ while $\DJS(P \parallel Q)>0$.
\end{proposition}

\begin{wrapfigure}{r}{0.48\textwidth}
\vspace{-8pt}
\centering
\includegraphics[width=0.47\textwidth]{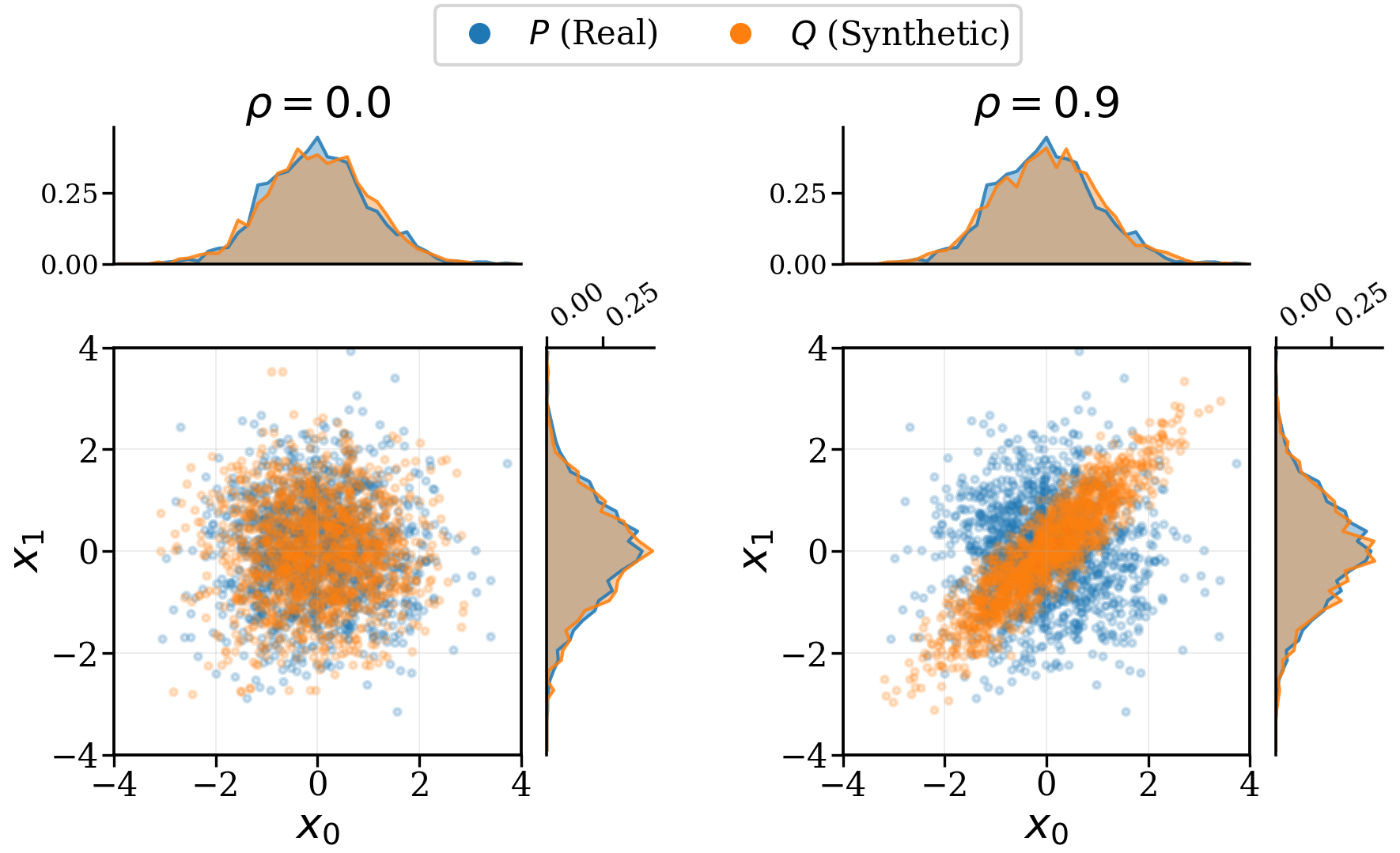}
\caption{Marginals are preserved while the joint dependence structure changes, causing marginal-based estimators to miss the discrepancy.}
\label{fig:marginals}
\vspace{-10pt}
\end{wrapfigure}
A simple construction uses two binary variables that are perfectly correlated under $P$ and perfectly anti-correlated under $Q$; see Appendix \ref{app:limitations_marginals}. Consequently, marginal estimators do not, in general, recover the joint divergence $\mathrm{D_{JS}}(P \parallel Q)$. Rather, they should be interpreted as marginal-fidelity metrics: a low marginal $\DJS$ may reflect accurate univariate statistics while masking substantial discrepancies in the dependence structure. Figure \ref{fig:marginals} illustrates this limitation: although $P$ and $Q$ share identical marginals, differences in correlation $\rho$ alter their joint structure, which marginal-based estimators fail to capture.

\subsection{Classifier-based divergence estimation}  
\label{subsec:Classifier-based Divergence Estimation}
Estimating $\DJS$ over the joint distribution without densities can be reduced to a density ratio estimation problem \cite{sugiyama2012density}. Consider a binary classification problem where samples from $P$ are labeled $y=1$ and samples from $Q$ are labeled $y=0$ (see Figure \ref{fig:js_classifier}). The Bayes-optimal classifier is given by
\begin{equation}
D^*(x) = \mathbb{P}(y=1 \mid x) = \frac{p(x)}{p(x)+q(x)} 
\label{eq:bayes_optimal}
\end{equation}
which implicitly encodes the density ratio $p(x)/q(x)$ through its posterior probabilities. Substituting this expression into the definition of $\DJS$ yields the representation
\begin{equation}
\mathrm{D_{JS}}(P \parallel Q) = \frac{1}{2}\mathbb{E}_{x\sim P}
\left[\log 2D^*(x)\right] + \frac{1}{2}\mathbb{E}_{x\sim Q}
\left[\log 2(1-D^*(x))\right],
\label{eq:js_e}
\end{equation}
establishing a direct link between probabilistic classification and divergence estimation \cite{goodfellow2014generative}. This formulation enables the estimation of $\DJS$ without explicitly modeling the underlying densities \cite{nguyen2010estimating}. Details on the derivation are provided in Appendix \ref{app:Density}. In practice, $D^*(x)$ is approximated by a parametric classifier $\hat{D}(x)$ trained on labeled samples. \citet{apellaniz2024divergence} showed the effectiveness of this approach using a fixed Multilayer Perceptron (MLP) architecture across controlled and real-world settings. However, their analysis is restricted to a single classifier family and does not examine how reliability varies with model choice, sample size, class imbalance, or dimensionality. 

\begin{wrapfigure}{r}{0.48\textwidth}
\vspace{-8pt}
\centering
\resizebox{0.46\textwidth}{!}{%
\begin{tikzpicture}[
    box/.style={draw, rectangle, rounded corners, align=center,
                minimum width=2.5cm, minimum height=1cm},
    arr/.style={->, thick}
]
\node[box] (p) at (0,0.8) {Samples from $P(x)$};
\node[box] (q) at (0,-0.8) {Samples from $Q(x)$};
\node[box] (clf) at (4,0) {Probabilistic\\classifier};
\node[box] (js) at (8,0) {$\widehat{\mathrm{D}}_{JS}(P \parallel Q)$};
\draw[arr] (p.east) -- node[pos=0.55, above=7pt] {$y=1$} (clf.west);
\draw[arr] (q.east) -- node[pos=0.55, below=7pt] {$y=0$} (clf.west);
\draw[arr] (clf.east) -- node[above] {$\hat{D}(x)$} (js.west);
\end{tikzpicture}%
}
\caption{Classifier-based $\DJS$ estimation. Samples from $P(x)$ and $Q(x)$ are labeled to train a probabilistic classifier, whose output $\hat{D}(x)$ estimates the divergence.}
\label{fig:js_classifier}
\vspace{-8pt}
\end{wrapfigure}
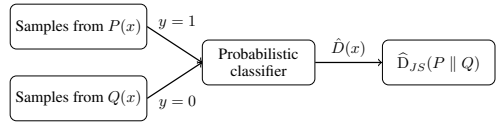
\paragraph{Target quantity versus estimation error.} We distinguish two error sources. Marginal-based estimators commit a \emph{target mismatch}: they do not approximate $\mathrm{D_{JS}}(P \parallel Q)$ but instead average univariate divergences, ignoring inter-variable dependencies. Classifier-based estimators target the true joint divergence $D_{\mathrm{JS}}^{\mathrm{joint}}(P \parallel Q)$, a well-defined function of the underlying distributions, but introduce an \emph{estimation error}: in finite-sample regimes, $\widehat{D}_{\mathrm{JS}}$ depends on the model, its inductive bias, and calibration. Thus, marginal methods estimate the wrong quantity, while classifier-based methods estimate the right quantity with an imperfect instrument.

\paragraph{Sources of sensitivity in classifier-based divergence estimation.}
\label{subsec:sources}
Classifier-based approaches estimate divergences over high-dimensional joint distributions, but their accuracy depends on the estimator and the underlying distributions. We consider four sources of sensitivity:

\begin{itemize}
    \item \textbf{Finite-sample effects.} Estimating the Bayes-optimal discriminator from finite data introduces variance and bias. Divergence estimation is known to be unreliable in small-sample regimes \cite{menendez1997js}, and neural estimators may be high-variance under limited data \cite{belghazi2018mutual}.

    \item \textbf{Classifier capacity and inductive bias.} Classifier families approximate decision boundaries differently. Underparameterized models may underestimate divergence, whereas high-capacity discriminators may saturate or overfit sampling artifacts \cite{arjovsky2017towards}. This bias--variance trade-off is central to discriminator-based $f$-divergence estimation \cite{nguyen2010estimating, nowozin2016f}.

    \item \textbf{High dimensionality.} Density-ratio estimation becomes harder as dimensionality grows \cite{perez2008kullback}. High dimensions amplify sparsity and sample complexity, leading to high-variance or biased estimates \cite{beyer1999nearest, verleysen2005curse, donahue2016adversarial, theis2015note}.

    \item \textbf{Class imbalance.} Unequal sampling from $P$ and $Q$ changes the class priors seen by the discriminator, inducing a prior shift in its posterior outputs. This biases the resulting $\DJS$ estimate unless corrected, motivating the ratio correction introduced in Section \ref{subsec:prior_correction} \cite{menendez1997js, sugiyama2012density}.
\end{itemize}

\subsection{Related work}
\label{subsect:related_works}

Synthetic tabular data evaluation has received increasing attention in generative modeling and privacy-preserving data generation \cite{stenger2024evaluation}. Existing frameworks use metrics ranging from statistical similarity to downstream utility \cite{rashidian2020smooth, snoke2018general}, yet no standard metric fully captures distributional similarity between real and synthetic 
data \cite{hernadez2023synthetic}. Many frameworks assess fidelity through marginal divergence criteria. Libraries such as \textit{Synthcity} \cite{qian2023synthcity} and \textit{Syndat} \cite{syndat} approximate $\DJS$ by averaging feature-wise discrepancies across dimensions. Despite their simplicity, this approach assumes feature independence and misses joint discrepancies, motivating our work. Alternative joint-structure metrics include Maximum Mean Discrepancy (MMD) \cite{gretton2012kernel} and Wasserstein distance \cite{arjovsky2017wasserstein}. However, both degrade in high dimensions: MMD is sensitive to kernel and bandwidth selection \cite{nguyen2021distributional, harder2021dp}, while the Wasserstein distance suffers slow convergence and high computational cost \cite{weed2019sharp, genevay2018learning}.
 
A complementary line estimates divergences from samples using probabilistic classifiers. Classifier-based two-sample tests (C2ST) \cite{lopez2016C2ST} and density ratio estimation methods \cite{sugiyama2012density} reduce the distribution comparison to a binary classification problem via the connection between optimal discriminators and density ratios. In particular, \citet{goodfellow2014generative} showed that under an optimal discriminator, the GAN objective minimizes an affine transformation of $\DJS$, enabling divergence estimation without explicit density modeling. More recently, \cite{apellaniz2024divergence} shows that MLPs can effectively approximate density ratios for evaluating synthetic tabular data. These approaches fit within $f$-divergence estimation, which provides a unifying theoretical perspective on divergence estimation from samples \cite{nguyen2010estimating, nowozin2016f}. Despite their flexibility, classifier-based $\DJS$ estimators remain insufficiently characterized across models and data regimes. The effects of model capacity, sample size, and dimensionality on the estimate are understudied for synthetic tabular data evaluation \cite{arjovsky2017towards, gulrajani2017improved}.

\section{Estimator-aware Jensen-Shannon divergence estimation}
\label{sec:methodology}

\subsection{Empirical Jensen-Shannon divergence as a measurement protocol}
\label{subsec:measurement_protocol}

A reported empirical $\DJS$ value is not a direct observation of an intrinsic distributional property, but the output of an estimation protocol. We make this explicit by denoting the empirical divergence as $\widehat{D}_{\mathrm{JS}}^{\,\mathcal{A}}(P,Q),$ where $\mathcal{A} = (\mathcal{F}, \mathcal{S}, \mathcal{C}, r, \epsilon)$ specifies the estimation protocol. $\mathcal{F}$ denotes the classifier family, $\mathcal{S}$ the train/validation/test split, $\mathcal{C}$ the probabilistic scoring or calibration strategy, $r=|Q_{\mathrm{train}}|/|P_{\mathrm{train}}|$ the empirical class-prior ratio, and $\epsilon$ the probability clipping threshold used for numerical stability (see Appendix \ref{app:numerical-stability}). This notation emphasizes that empirical $\DJS$ values are protocol-dependent measurements: two practitioners using different protocols $\mathcal{A}$ and $\mathcal{A}'$ may obtain different values, even when evaluating the same pair of datasets.

Given datasets sampled from $P$ and $Q$, our goal is to estimate $\mathrm{D_{JS}}(P \parallel Q)$ without access to the underlying densities. Algorithm \ref{alg:estimator_aware_djs} summarizes the estimator-aware protocol used throughout the paper. The classifier-based estimator is formalized in Section \ref{subsec:class_js}, the prior correction under class imbalance is derived in Section \ref{subsec:prior_correction}, and the estimator families considered are described in Section \ref{subsec:classifiers}. In all experiments, protocol components such as the data split, hyperparameter selection criterion, correction rule, clipping value, random seeds, and runtime are fixed or explicitly reported to make empirical $\DJS$ values reproducible and comparable.

\begin{algorithm}[ht]
\caption{Estimator-aware classifier-based $\DJS$ estimation}
\label{alg:estimator_aware_djs}
\begin{algorithmic}[1]
\Require $X_P \sim P$, $X_Q \sim Q$; classifier family $\mathcal{F}$; threshold $\tau$; clipping $\epsilon$
\State $\mathcal{D} \leftarrow \{(x,1): x\in X_P\}\cup\{(x,0): x\in X_Q\}$
\State Split $\mathcal{D}$ into train/validation/test sets
\State Fit $\hat{D}_{\mathcal{F}}(x)\approx \mathbb{P}(y=1\mid x)$ and tune hyperparameters (if applicable)
\State Obtain $\hat{D}_{\mathcal{F}}(x)$ on the test set
\State Let $r$ be the empirical class-prior ratio from Equation~\ref{eq:class_prior_ratio}
\If{$|r-1|>\tau$}
    \State Set $\hat{D}_{\mathrm{corr}}(x)$ by applying Equation~\ref{eq:prior_correction} to $\hat{D}_{\mathcal{F}}(x)$
\Else
    \State $\hat{D}_{\mathrm{corr}}(x)\leftarrow \hat{D}_{\mathcal{F}}(x)$
\EndIf
\State $\widetilde{D}_{\mathcal{F}}(x)\leftarrow
\mathrm{clip}\!\left(\hat{D}_{\mathrm{corr}}(x),\epsilon,1-\epsilon\right)$
\State Compute $\widehat{D}_{\mathrm{JS}}^{\mathcal{A}}(P,Q)$ using Equation~\ref{eq:class_js_estimator}
\State Report $\widehat{D}_{\mathrm{JS}}^{\mathcal{A}}$, uncertainty across seeds, runtime, and protocol details
\end{algorithmic}
\end{algorithm}

\subsection{Classifier-based joint Jensen-Shannon divergence estimator}
\label{subsec:class_js}

Using Equation \ref{eq:js_e}, we estimate $\DJS$ using the posterior output of a probabilistic discriminator. For a classifier family $\mathcal{F}$, let $\hat{D}_{\mathcal{F}}(x) \approx \mathbb{P}(y=1 \mid x)$ denote the estimated probability that a sample comes from $P$. After applying prior correction when needed and clipping for numerical stability, we denote the posterior used for evaluation by $\widetilde{D}_{\mathcal{F}}(x)$. The classifier-based estimator is
\begin{equation}
\widehat{D}_{\mathrm{JS}}^{\mathcal{A}}(P \parallel Q)
= \frac{1}{2}\mathbb{E}_{x\sim P}
\left[\log\big(2\widetilde{D}_{\mathcal{F}}(x)\big)\right]
+ \frac{1}{2}\mathbb{E}_{x\sim Q}
\left[\log\big(2(1-\widetilde{D}_{\mathcal{F}}(x))\big)\right].
\label{eq:class_js_estimator}
\end{equation}
The posterior $\widetilde{D}_{\mathcal{F}}(x)$ corresponds to the classifier output after prior correction when needed and clipping to $[\epsilon,1-\epsilon]$ for numerical stability (see Appendix \ref{app:numerical-stability}). When $\widetilde{D}_{\mathcal{F}}(x)$ matches the balanced Bayes-optimal discriminator, this estimator recovers the population $\DJS$. In finite samples, deviations from this ideal depend on the classifier family, calibration quality, and data regime. Importantly, $\widehat{D}_{\mathrm{JS}}^{\mathcal{A}}$ depends on calibrated posterior values, not only on the induced decision boundary: classifiers with similar two-sample accuracy may yield different divergence estimates if their posteriors differ in calibration or sharpness \cite{lopez2016C2ST}. Thus, probabilistic scoring and calibration are part of the measurement protocol rather than implementation details.

\subsection{Prior-shift correction under imbalanced sampling}
\label{subsec:prior_correction}

The estimator in Equation \ref{eq:class_js_estimator} implicitly assumes balanced sampling from $P$ and $Q$. In practice, the labeled training data may be imbalanced, inducing a prior shift in the learned discriminator. Let
\begin{equation}
\pi_1 = P(y=1), \quad \pi_0 = P(y=0),
\qquad
r = \frac{\pi_0}{\pi_1}
= \frac{|Q_{\mathrm{train}}|}{|P_{\mathrm{train}}|},
\label{eq:class_prior_ratio}
\end{equation}
where $\pi_1$ and $\pi_0$ denote the class priors, and $r$ is the empirical class-prior ratio.

Under balanced sampling ($\pi_1 = \pi_0$), the Bayes-optimal discriminator satisfies Equation \ref{eq:bayes_optimal}, which matches the symmetric form underlying the $\DJS$ estimator. Under class imbalance, Bayes' rule yields
\begin{equation}
D^*(x) = \frac{p(x)\,\pi_1}{p(x)\,\pi_1 + q(x)\,\pi_0}
= \frac{p(x)}{p(x) + r\, q(x)},
\end{equation}
Thus, the uncorrected estimator no longer recovers the balanced $\DJS$ objective. We correct this prior shift with a post-hoc adjustment to the discriminator output. Here, $D(x)$ denotes a generic posterior; in Algorithm \ref{alg:estimator_aware_djs}, it corresponds to $\hat{D}_{\mathcal{F}}(x)$. Given a discriminator posterior $D(x)$ trained under priors $(\pi_1,\pi_0)$, the corrected posterior corresponding to balanced priors is
\begin{equation}
D_{\mathrm{corr}}(x)
=
\frac{\frac{D(x)}{\pi_1}}
{\frac{D(x)}{\pi_1}+\frac{1-D(x)}{\pi_0}}
=
\frac{rD(x)}{1+(r-1)D(x)}.
\label{eq:prior_correction}
\end{equation}
The last equality follows from substituting $\pi_0=r\pi_1$. Under well-calibrated posteriors, the correction maps the discriminator output learned under imbalanced priors to the balanced-prior posterior required by the $\DJS$ estimator. It does not correct for model misspecification, finite-sample error, or posterior miscalibration, and should therefore be interpreted as a necessary protocol adjustment rather than a guarantee of estimator consistency. We evaluate its impact in Section \ref{sec:results}.
 
\subsection{Estimator families and protocol components}
\label{subsec:classifiers} 
We select estimator families to span different inductive biases and capacity regimes. The goal is not to identify a universally best classifier, but to characterize how $\widehat{D}_{\mathrm{JS}}^{\mathcal{A}}$ changes when the density ratio is approximated by different model classes. We include linear, tree-based, neural, and pretrained models, covering different levels of capacity, computational cost, and suitability for tabular data. Tree-based models and neural networks are widely used competitive approaches for tabular classification \cite{mcelfresh2023neural}.

\begin{itemize}
    \item \textbf{Linear models}: logistic regression (LogReg) and polynomial logistic regression (LR-Pol). While both are low-capacity estimators, LR-Pol scales combinatorially with dimensionality and degree ($\mathcal{O}(d^k)$), making it computationally expensive in high-dimensional regimes.
    \item \textbf{Tree-based models}: random forests (RF) and gradient boosting (XGBoost), which capture nonlinear decision boundaries with moderate complexity.
    \item \textbf{Neural networks}: multilayer perceptrons (MLP), which provide high-capacity function approximation and are commonly used as flexible discriminators.
    \item \textbf{Pretrained models}: TabPFN \cite{hollmann2022tabpfn}, a pretrained in-context learner with a distinct inductive bias and strong performance in small-data regimes. TabPFN was originally designed for datasets with fewer than 1,000 samples and 100 features.
\end{itemize} 

\subsection{Reporting and reliability diagnostics}
\label{subsec:reporting_protocol}

Because empirical $\DJS$ values are protocol-dependent, reporting a scalar divergence without its estimation protocol is incomplete. We report each estimate with the protocol components that determine it: classifier family, data split, sampling ratio, correction rule, probability clipping, hyperparameter selection, random seeds, and runtime. In controlled experiments, where a Monte Carlo (MC) reference $\mathrm{D_{JS}}^{\star}$ is available, reliability is assessed by the absolute estimation error $ \left|\widehat{D}_{\mathrm{JS}}^{\mathcal{A}}(P,Q)-\mathrm{D_{JS}}^{\star}\right|.$ In real-world experiments, where the true divergence is unavailable, reliability is assessed through seed consistency and disagreement across estimator families. Large cross-estimator disagreement indicates protocol sensitivity: the reported divergence should not be interpreted as a standalone property of the synthetic data, but also as a consequence of the measurement protocol.

\section{Empirical Evaluation}
\label{sec:results}

\subsection{Evaluation protocol}
\label{sec:exp} 

We evaluate empirical $\DJS$ estimation across three controlled failure regimes and two real-world benchmarks, treating estimation as a measurement protocol rather than a standalone classifier benchmark. The controlled settings isolate target mismatch in marginal estimators, prior-shift bias under class imbalance, and estimator sensitivity under increasing dimensionality (see Section \ref{subsec:sources}).

\begin{table}[ht]
\centering
\scriptsize
\setlength{\tabcolsep}{2pt}
\small
\caption{Empirical settings. $M$ and $L$ denote training and evaluation samples. Variable parameters: $\rho$ (dependence shift), $|Q|/|P|$ (imbalance), \emph{gap} (Gaussian mean separation), $d$ (dimensionality), and dataset/generator (real-world).}
\label{tab:experiments_c}
\begin{tabular}{ccc}
\toprule
 \textbf{Evaluation setting} & \textbf{Variable parameter} & \textbf{Fixed setup} \\
\midrule
 Target mismatch in marginal estimators & $\rho \in \{0.0,\dots,0.9\}$ & $M=L=2{,}000$\\
 Class imbalance  & $|Q|/|P| \in \{0.1,\dots,1.0\}$ & $M=L=2{,}000$, \emph{gap} $\in \{0.3,0.7,1.0\}$ \\
High-dimensional sensitivity & $d \in \{2,10,25,40,50\}$ & $M=L=2{,}000$ \\
Real-world benchmarks 
& Dataset / generator 
& High-sample; Low-sample \\
\bottomrule
\end{tabular}
\end{table}

Table \ref{tab:experiments_c} summarizes all configurations. In all controlled settings $D_{\mathrm{JS}}^{\star}$ is approximated via MC (Appendix \ref{app:mc}), and accuracy is measured by the mean absolute error (MAE) between $\widehat{D}_{\mathrm{JS}}$ and $D_{\mathrm{JS}}^{\star}$. For each configuration, $M$ samples per distribution are used for training; an independent set of $2L$ samples is split evenly into validation and test subsets of size $L$ each. The validation subset is used for hyperparameter and model selection, while the test data are used only for the reported estimate. Controlled experiments use $M=L=2{,}000$. Real-world benchmarks use two sample-budget regimes with fixed train/validation/test proportions: high-sample $H=(N,M,L)=(10k,7.5k,1k)$ and low-sample $L_s=(1k,750,100)$. Ratio correction is applied when the empirical class-prior ratio deviates from balance ($|r - 1| > 0.1$). All experiments are repeated over 5 random seeds. Implementation details are provided in Appendix \ref{app:hyperparameter_tuning}.

\subsection{Marginal estimators suffer from target mismatch}
\label{subsec:marginal}

\begin{wrapfigure}{r}{0.48\textwidth}
\vspace{-8pt}
\centering
\includegraphics[width=0.46\textwidth]{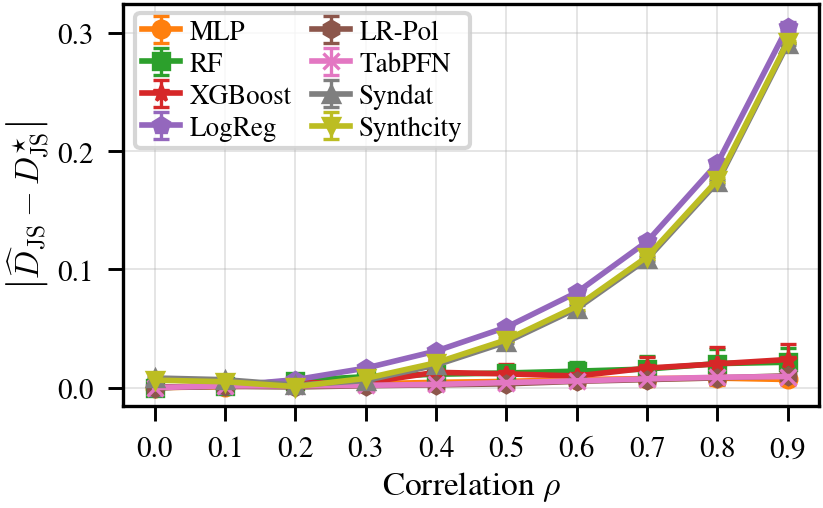}
\caption{Absolute $\DJS$ estimation error versus correlation $\rho$. Marginal estimators show increasing error as joint divergence grows, while most classifier-based estimators remain accurate.}
\label{fig:js_uc5}
\vspace{-8pt}
\end{wrapfigure}

We consider bivariate Gaussian distributions with fixed marginals and varying correlation $\rho \in \{0.0,\dots,0.9\}$ (Figure \ref{fig:marginals}). Since only the joint distribution varies with $\rho$, this setting isolates whether estimators capture dependence structure. Figure \ref{fig:js_uc5} reports absolute error relative to the MC reference estimation. Marginal methods (Syndat, Synthcity) become increasingly inaccurate as $\rho$ grows, confirming that marginal agreement does not imply joint fidelity. Most classifier-based estimators maintain low error across correlations, with MLP, RF, XGBoost, LR-Pol, and TabPFN remaining close to the reference MC estimation. In contrast, linear LogReg deteriorates at large $\rho$, reflecting its limited capacity to model nonlinear dependence shifts. These results show that marginal estimators suffer from target mismatch, while classifier-based estimators can recover joint divergence provided the classifier family is sufficiently expressive. Detailed results are in Appendix \ref{app:marginal_table}.

\subsection{Prior correction is necessary under imbalanced sampling}
\label{subsec:results_imbalance}
We generate pairs of bivariate Gaussian distributions with controlled mean separation. Specifically, we set the distance between the means of $P$ and $Q$ to $\mathrm{gap} \in \{0.3, 0.7, 1.0\}$, so that larger gap values correspond to more easily distinguishable distributions and larger MC reference estimation. For each gap value, we vary the sampling ratio $r=|Q|/|P| \in \{0.1,\dots,1.0\}$, evaluating sensitivity to prior shift and the effectiveness of ratio correction. Additional details and visualizations of the class-imbalance and gap setups are provided in Appendix \ref{app:class_imbalance_setup} (Figures \ref{fig:uc6_imbalance} and \ref{fig:uc6_gap_scenarios}).

\begin{figure}[ht]
\centering
\includegraphics[width=1.0\linewidth]{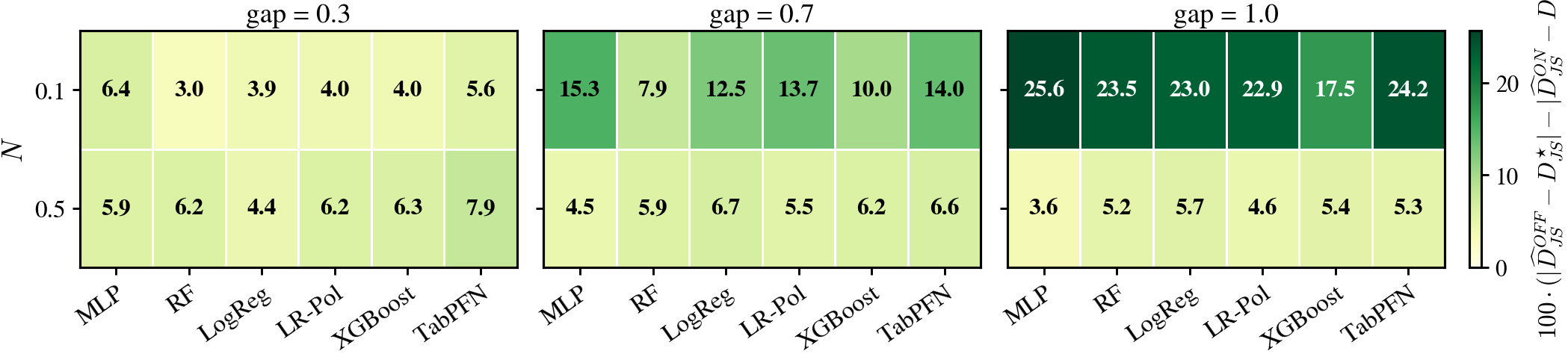}
\caption{Effect of prior correction under class imbalance. Each cell reports the reduction in absolute Jensen--Shannon estimation error from the corrected estimator over the uncorrected one:
$100\cdot(|\widehat{D}_{\mathrm{JS}}^{\mathrm{OFF}}-D_{\mathrm{JS}}^\star|
-
|\widehat{D}_{\mathrm{JS}}^{\mathrm{ON}}-D_{\mathrm{JS}}^\star|)$.
Positive values indicate improvement. Results are shown for different gaps, imbalance ratios $N=|Q|/|P|$, and classifiers.}
\label{fig:uc13_imbalance_improvement}
\end{figure}

Figure \ref{fig:uc13_imbalance_improvement} reports the gain in absolute estimation error from prior correction across classifiers, imbalance ratios $N \in \{0.1,0.5\}$, and distributional gaps. Correction yields the largest gains in the most challenging regime ($N=0.1$), especially at large gap values, where uncorrected classifiers are most dominated by class priors and divergence estimates are least informative. As $N \to 1$, the prior shift becomes negligible and the corrected and uncorrected estimates converge. Overall, these results validate Section \ref{subsec:prior_correction}: under strong prior shift, classifier-based $\DJS$ estimates become prior-dominated, while the proposed correction removes this posterior bias under calibrated posteriors and recovers meaningful estimates. The correction is specific to classifier-based estimators; for marginal estimators, it has no effect because their estimates are computed from feature-wise divergences rather than classifier posteriors. Full numerical results are reported in Table \ref{tab:uc13_on_abs_error_x100_all_gaps_mean_std} in Appendix \ref{app:results_imbalance}.

\subsection{High-dimensional estimation exposes accuracy-cost trade-offs}
\label{subsec:results_highdim}

We analyze the effect of increasing dimensionality using multivariate Gaussian distributions with $d \in \{2,10,25,40,50\}$. Sample sizes are fixed at $M=L=2{,}000$ across dimensions, so larger $d$ also makes the problem increasingly sample-limited relative to the feature dimension. Figure \ref{fig:uc7} reports estimation accuracy and runtime as dimensionality increases. The left panel shows that classifier-based estimators generally track the MC reference estimation $D_{\mathrm{JS}}^{\star}$, whereas marginal baselines (Syndat, Synthcity) fail to capture the growth in joint divergence. Their estimates remain nearly constant across dimensions, causing errors to increase with $D_{\mathrm{JS}}^{\star}$. At $d=50$, the best classifier-based estimators remain substantially more accurate than the marginal baselines: MLP and TabPFN achieve the lowest errors, while RF exhibits the largest bias among classifier-based methods. Syndat and Synthcity reach errors roughly five times larger. The right panel summarizes the runtime-accuracy trade-off at $d=50$. LogReg and TabPFN offer the best accuracy-efficiency balance, achieving low error at low computational cost. MLP and XGBoost are more expensive, while LR-Pol is the most computationally demanding. This is expected because polynomial feature expansion scales combinatorially with dimensionality and degree, $\mathcal{O}(d^k)$, leading to runtimes above 29 minutes at $d=50$. Thus, LR-Pol becomes impractical beyond moderate dimensionality, with performance constrained by computational budget rather than statistical capacity. Full results are in Appendix \ref{app:results_dimension}.

\begin{figure}[ht]
\centering
\includegraphics[width=0.9\linewidth]{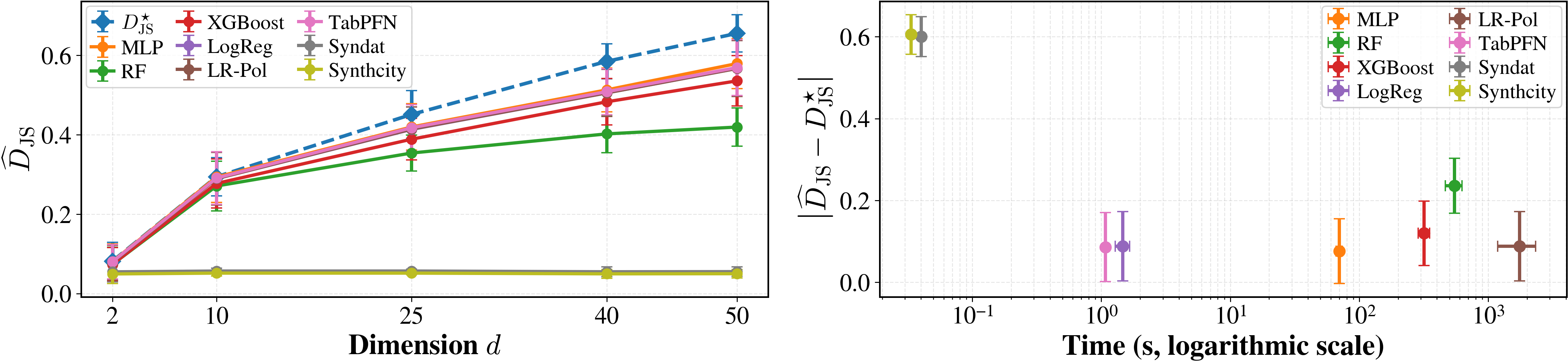}
\caption{$\DJS$ estimation across dimensions and runtime--accuracy trade-off. 
Left: $\widehat{D}_{\mathrm{JS}}$ versus dimensionality, compared with the MC reference estimation $D_{\mathrm{JS}}^{\star}$; error bars show one standard deviation across seeds. 
Right: runtime versus MAE at $d=50$, with time on a logarithmic scale.}
\label{fig:uc7}
\end{figure}

\subsection{Real-world synthetic benchmarks reveal protocol-dependent conclusions} 
\label{subsec:real}  

To complement controlled experiments, we evaluate divergence estimation in realistic finite-sample settings without MC reference estimation. We use two tabular benchmarks, \textit{Adult} \cite{becker1996adult} and \textit{Intrusion} \cite{Stolfo1999intrusion}, with synthetic data from VAE \cite{kingma2013vae} and CTGAN \cite{xu2019ctgan}, following \cite{apellaniz2025artificial}. We compare estimates across classifier families, generators, and two sample budgets with fixed train/validation/test proportions: high ($N=10{,}000$, $M=7{,}500$, $L=1{,}000$) and low ($N=1{,}000$, $M=750$, $L=100$). 

\begin{figure}[ht]
\centering
\includegraphics[width=1.0\linewidth]{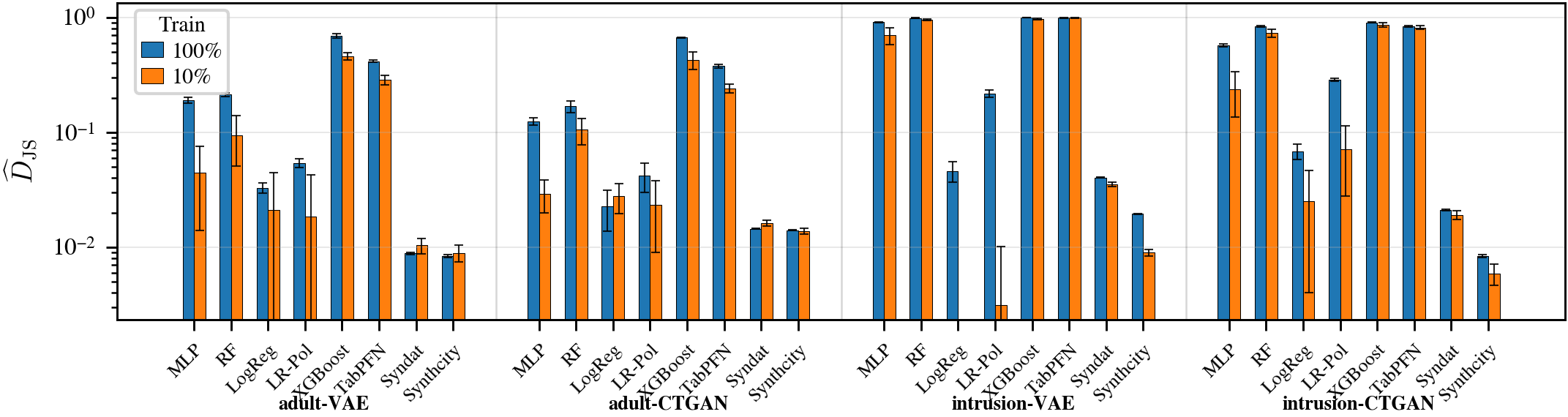}
\caption{Protocol-dependent $\DJS$ estimates on Adult and Intrusion under VAE and CTGAN, in high- and low-sample regimes. Error bars show one standard deviation; log scale highlights order-of-magnitude differences.}
\label{fig:real_case} 
\end{figure}

Figure \ref{fig:real_case} shows that $\DJS$ estimates are strongly protocol-dependent: across dataset--generator pairs, estimators can differ by orders of magnitude. For instance, on \textit{Intrusion} with VAE, LogReg reports much smaller divergence than high-capacity estimators such as XGBoost and TabPFN, suggesting that low-capacity models underestimate nonlinear multivariate discrepancies. Marginal methods (Syndat, Synthcity) also report lower values, consistent with their insensitivity to joint-distribution discrepancies in controlled experiments. Comparing high- and low-sample regimes reveals robustness to reduced data: some methods preserve similar estimates, whereas others show larger shifts or variance. Overall, estimator choice dominates variability in real-world $\DJS$ evaluation; reported divergence values are incomplete without the estimator specification, protocol, calibration, and data regime. Full numerical results are reported in Appendix \ref{app:results_real}.

\section{Conclusions}
\label{sec:conclusions}
Empirical $\DJS$ estimation is protocol-dependent: classifier choice can shape reported values, making an empirical divergence estimate difficult to interpret without full specification of the estimation pipeline. Through controlled experiments, we identified three key failure regimes (dependence structure, class imbalance, and high dimensionality) and showed that no single estimator dominates across all settings. Marginal-based approaches underestimate divergence by ignoring joint dependencies, while 
classifier-based estimators vary significantly across model families. We also provide prior correction and practical estimator guidelines.

\subsection{Guide to practitioners}
Table \ref{tab:guidelines} summarizes estimator recommendations across regimes. MLP and TabPFN offer the best overall accuracy but differ markedly in cost; LogReg is fastest but cannot detect nonlinear structure; LR-Pol is expressive but becomes computationally prohibitive at high dimensionality. Marginal-based approaches (Syndat, Synthcity) should be avoided whenever joint dependencies matter, which is the norm in structured tabular data.

\begin{table}[ht]
\centering
\small
\caption{Guidelines for classifier-based $\DJS$ estimation across data regimes. \textbf{\cmark}: recommended; \textbf{\tmark}: acceptable with caveats; \textbf{\xmark}: not recommended. For \emph{Computational cost}, \textbf{\cmark} denotes low cost.}
\label{tab:guidelines}
\setlength{\tabcolsep}{5pt}
\begin{tabular}{lcccccccc}
\toprule
\textbf{Setting} & \textbf{MLP} & \textbf{LogReg} & \textbf{LR-Pol} 
& \textbf{RF} & \textbf{XGBoost} & \textbf{TabPFN} 
& \textbf{Syndat} & \textbf{Synthcity} \\
\midrule
Dependence structure  
& \cmark & \xmark & \cmark & \cmark & \cmark & \cmark & \xmark & \xmark \\
High dimensionality   
& \cmark & \tmark & \xmark & \tmark & \cmark & \tmark & \xmark & \xmark \\
Class imbalance       
& \cmark & \tmark & \cmark & \tmark & \tmark 
& \tmark & \tmark & \tmark \\
Finite samples ($d$ small)       
& \tmark & \cmark & \tmark & \tmark & \tmark & \cmark & \tmark & \tmark \\
Computational cost    
& \xmark & \cmark & \xmark & \tmark & \tmark 
& \cmark & \cmark & \cmark \\
\bottomrule
\end{tabular}
\end{table}

\subsection{Limitations and future work}
\label{subsec:limitations}

Our controlled experiments rely on Gaussian distributions, which may not reflect the complexity of real tabular data (mixed types, heavy tails, missing values). The analysis is primarily empirical: we provide no theoretical guarantees on consistency or convergence rates, and conclusions are limited to the classifier families and regimes considered. The imbalance correction assumes the ratio $r$ is known, which may introduce additional error in practice.

Several directions remain open. On the theoretical side, bias-variance characterizations of classifier-based $\DJS$ estimators across data regimes would complement our empirical findings. On the practical side, incorporating additional model families (kernel methods, modern tabular transformers, calibrated ensembles) into the pipeline is straightforward given its classifier-agnostic design. Extending the analysis to alternative divergences (Wasserstein, MMD) would clarify whether estimator sensitivity is specific to $\DJS$ or reflects a 
broader phenomenon. Finally, applying this framework beyond tabular data, to images or time series, would test the generality of the observed failure modes.

\begin{ack}   
SYNTHIA (Synthetic Data Generation framework for integrated validation of use cases and AI healthcare applications) is supported by the Innovative Health Initiative Joint Undertaking (IHI JU) under grant agreement No 101172872. The JU receives support from the European Union's Horizon Europe research and innovation programme, COCIR, EFPIA, Europa Bío, MedTech Europe, Vaccines Europe and DNV. The UK consortium partner, The National Institute for Health and Care Excellence (NICE) is supported by UKRI Grant 10132181. Funded by the European Union, the private members, and those contributing partners of the IHI JU. Views and opinions expressed are however those of the author(s) only and do not necessarily reflect those of the aforementioned parties. Neither of the aforementioned parties can be held responsible for them.

\end{ack}

{
\small
\bibliographystyle{unsrtnat}
\bibliography{references}
}

\appendix

\section{Technical appendices and supplementary material}

\subsection{Jensen-Shannon divergence estimation}
\label{app:Density}

Estimating the divergence between two probability distributions $P$ and $Q$ can be reduced to estimating their density ratio. Let $p(x)$ and $q(x)$ denote the corresponding densities. The density ratio is defined as
\begin{equation}
r^*(x) = \frac{p(x)}{q(x)}.
\end{equation}

Direct estimation of $p(x)$ and $q(x)$ is often difficult in high-dimensional settings. Density ratio estimation provides an alternative by learning $r^*(x)$ directly from samples without requiring explicit density estimation \cite{sugiyama2012density}.

\subsubsection{Jensen-Shannon divergence}

Let $P$ and $Q$ be probability distributions on $\mathcal{X} \subseteq \mathbb{R}^d$, with densities $p(x)$ and $q(x)$ (when they exist). Define the mixture distribution
\begin{equation}
M = \frac{1}{2}(P + Q), 
\quad
m(x) = \frac{1}{2}\big(p(x) + q(x)\big).
\end{equation}

The $\DJS$ is defined as
\begin{equation}
\mathrm{D_{JS}}(P \parallel Q) = \frac{1}{2} \mathrm{D_{KL}}(P \parallel M) + \frac{1}{2} \mathrm{D_{KL}}(Q \parallel M).
\end{equation}

Equivalently,
\begin{equation}
\mathrm{D_{JS}}(P \parallel Q)
=
\frac{1}{2} \mathbb{E}_{x \sim P} \left[ \log \frac{p(x)}{m(x)} \right]
+
\frac{1}{2} \mathbb{E}_{x \sim Q} \left[ \log \frac{q(x)}{m(x)} \right].
\end{equation}

\subsubsection{Limitations of marginal Jensen-Shannon divergence}
\label{app:limitations_marginals}

\textbf{Proof.}
Let \(X_1,X_2 \in \{0,1\}\). Define \(P\) by
\(P(0,0)=P(1,1)=1/2\), and \(Q\) by
\(Q(0,1)=Q(1,0)=1/2\). 
Both distributions have identical one-dimensional marginals:
\(P(X_i=0)=Q(X_i=0)=1/2\) and 
\(P(X_i=1)=Q(X_i=1)=1/2\) for \(i=1,2\). 
Therefore,
\[
D^{marg}_{JS}(P\|Q)
=
\frac{1}{2}
\sum_{i=1}^{2}
D_{JS}(P(X_i)\|Q(X_i))
=0.
\]
However, the joint supports of \(P\) and \(Q\) are disjoint. Hence the joint distributions are different, and with logarithms in base 2,
\[
D^{joint}_{JS}(P\|Q)=1.
\]
This shows that marginal Jensen-Shannon divergence can be zero even when the joint Jensen-Shannon divergence is maximal.

\subsubsection{Classifier-based estimation of Jensen-Shannon divergence}
\label{app:classification_js}

We briefly review the connection between probabilistic classification, density ratio estimation, and $\DJS$.

\paragraph{From classification to density ratios.}
Consider a binary classification problem where samples from $P$ are labeled $y=1$ and samples from $Q$ are labeled $y=0$. Let $D(x) = \mathbb{P}(y=1 \mid x)$ denote the posterior probability estimated by a classifier.

By Bayes' rule, the density ratio can be written as
\begin{equation}
r^*(x) = \frac{p(x)}{q(x)} 
= \frac{P(y=0)}{P(y=1)} \cdot \frac{D(x)}{1 - D(x)}.
\end{equation}

Under balanced class priors, this simplifies to
\begin{equation}
r^*(x) = \frac{D(x)}{1 - D(x)}.
\end{equation}

This establishes a direct link between probabilistic classification and density ratio estimation, as shown by \citet{apellaniz2024divergence}.

\paragraph{From density ratios to Jensen-Shannon divergence.}
The $\DJS$ can be expressed as
\begin{equation}
\mathrm{D_{JS}}(P \parallel Q)
=
\frac{1}{2} \mathbb{E}_{x \sim P}
\left[ \log \frac{2p(x)}{p(x) + q(x)} \right]
+
\frac{1}{2} \mathbb{E}_{x \sim Q}
\left[ \log \frac{2q(x)}{p(x) + q(x)} \right].
\end{equation}

Using the Bayes-optimal classifier
\begin{equation}
D^*(x) = \frac{p(x)}{p(x) + q(x)},
\end{equation}

we obtain the form
\begin{equation}
\mathrm{D_{JS}}(P \parallel Q)
=
\frac{1}{2} \mathbb{E}_{x \sim P}
\left[ \log \big(2 D^*(x)\big) \right]
+
\frac{1}{2} \mathbb{E}_{x \sim Q}
\left[ \log \big(2 (1 - D^*(x))\big) \right].
\end{equation}

In practice, $D^*(x)$ is approximated by a classifier $\hat{D}(x)$, yielding the estimator used in the main text.

\subsubsection{Monte Carlo estimation}
\label{app:mc}
In controlled settings where sampling from both distributions is possible, the $\DJS$ can be approximated using Monte Carlo integration. Given samples $\{x_i\}_{i=1}^L \sim p(x)$ and $\{\tilde{x}_i\}_{i=1}^L \sim q(x)$, we estimate

\begin{equation}
\mathrm{D}^{\star}_{JS}(P \,\|\, Q)
\approx
\frac{1}{2L} \sum_{i=1}^{L}
\log \frac{p(x_i)}{m(x_i)}
+
\frac{1}{2L} \sum_{i=1}^{L}
\log \frac{q(\tilde{x}_i)}{m(\tilde{x}_i)},
\end{equation}

where $m(x) = \frac{1}{2}(p(x) + q(x))$.

Equivalently, this can be expressed as
\begin{equation}
\mathrm{D}^{\star}_{JS}(P \,\|\, Q)
\approx
\frac{1}{2L} \sum_{i=1}^{L}
\log \left(
\frac{2p(x_i)}{p(x_i) + q(x_i)}
\right)
+
\frac{1}{2L} \sum_{i=1}^{L}
\log \left(
\frac{2q(\tilde{x}_i)}{p(\tilde{x}_i) + q(\tilde{x}_i)}
\right).
\end{equation}

This estimate serves as a reference approximation of the true divergence in synthetic experiments, allowing us to evaluate the accuracy of classifier-based estimators. This formulation enables estimating the $\DJS$ without explicit density estimation, relying solely on the discriminative power of the classifier.

\subsection{Hyperparameter search spaces}
\label{app:hyperparameter_tuning}

Since divergence estimation relies on calibrated probabilistic outputs \cite{guo2017calibration}, classifier hyperparameters are selected using Bayesian optimization with negative log-loss as the validation objective. For classifiers undergoing Bayesian hyperparameter search (RF, XGBoost, LogReg, and LR-Pol), posterior probabilities are post-hoc calibrated using
\texttt{CalibratedClassifierCV}. We use isotonic calibration when the
validation/calibration set contains more than 1000 samples and sigmoid
calibration otherwise. MLP and TabPFN are evaluated using their native
probabilistic outputs. 

We use \texttt{BayesSearchCV} \cite{head2021scikit} with 50 iterations and 5-fold cross-validation, repeating each experiment over 5 random seeds. Search spaces are summarized in Table \ref{tab:hyperparams}. In class-imbalance experiments, classifier-specific reweighting is disabled to isolate the explicit ratio correction: \texttt{class\_weight=None} for RF, LogReg, and LR-Pol, and \texttt{scale\_pos\_weight=1.0} for XGBoost. MLP and TabPFN are not tuned via Bayesian search: MLP uses Adam with early stopping and dropout, while TabPFN follows its standard inference protocol.

\begin{table*}[t]
\centering
\small
\caption{Hyperparameter search spaces for the considered classifiers. Bayesian optimization is used for RF, XGBoost, LogReg, and LR-Pol. MLP and TabPFN are trained without Bayesian hyperparameter search.}
\label{tab:hyperparams}
\begin{tabular}{lll}
\toprule
\textbf{Classifier} & \textbf{Hyperparameter} & \textbf{Search space} \\
\midrule
\multirow{7}{*}{Random Forest} 
& \texttt{n\_estimators} & $[100, 500]$ \\
& \texttt{max\_depth} & $[3, 10]$ \\
& \texttt{min\_samples\_leaf} & $[10, 100]$ \\
& \texttt{min\_samples\_split} & $[2, 20]$ \\
& \texttt{max\_features} & \{\texttt{sqrt}, \texttt{log2}\} \\
& \texttt{max\_samples} & $[0.5, 1.0]$ \\
& \texttt{class\_weight} & \{\texttt{balanced}, \texttt{balanced\_subsample}, \texttt{None}\} \\
\midrule
\multirow{8}{*}{XGBoost}
& \texttt{n\_estimators} & $[100, 600]$ \\
& \texttt{max\_depth} & $[3, 10]$ \\
& \texttt{learning\_rate} & $[10^{-3}, 3 \times 10^{-1}]$ (log-uniform) \\
& \texttt{subsample} & $[0.5, 1.0]$ \\
& \texttt{colsample\_bytree} & $[0.5, 1.0]$ \\
& \texttt{min\_child\_weight} & $[1, 20]$ \\
& \texttt{reg\_alpha} & $[10^{-8}, 10]$ (log-uniform) \\
& \texttt{reg\_lambda} & $[10^{-8}, 10]$ (log-uniform) \\
& \texttt{scale\_pos\_weight} & $[0.5, 2.5]$ \\
\midrule
\multirow{3}{*}{Logistic Regression}
& \texttt{C} & $[10^{-4}, 10^{2}]$ (log-uniform) \\
& \texttt{penalty} & \{\texttt{l1}, \texttt{l2}\} \\
& \texttt{class\_weight} & \{\texttt{None}, \texttt{balanced}\} \\
\midrule
\multirow{5}{*}{Polynomial Logistic Regression}
& \texttt{poly\_\_degree} & $[1, 3]$ \\
& \texttt{poly\_\_interaction\_only} & \{\texttt{False}, \texttt{True}\} \\
& \texttt{logreg\_\_C} & $[10^{-4}, 10^{2}]$ (log-uniform) \\
& \texttt{logreg\_\_penalty} & \{\texttt{l1}, \texttt{l2}\} \\
& \texttt{logreg\_\_class\_weight} & \{\texttt{None}, \texttt{balanced}\} \\
\midrule
MLP & -- & Trained directly with Adam and early stopping \\
TabPFN & -- & Direct fit, no Bayesian hyperparameter search \\
\bottomrule
\end{tabular}
\end{table*}

For the MLP, we use a three-layer fully connected architecture with 256, 64, and 32 hidden units and LeakyReLU activations. Dropout and early stopping are applied to mitigate overfitting. For LR-Pol, we account for the combinatorial growth of polynomial features with dimensionality and degree (i.e., $\mathcal{O}(d^k)$). To ensure computational tractability, candidate degrees are dynamically restricted based on the resulting feature dimensionality, discarding configurations that exceed a predefined threshold. In practice, this typically limits the search to low-degree polynomials (often degree $\leq 2$ in high-dimensional settings). When no candidate satisfies the constraint, the model defaults to degree 1. This design ensures a fair comparison across estimators under a fixed computational budget while preventing feature explosion.

\subsection{Computational resources}
\label{app:compute_resources}

All experiments were conducted on a workstation equipped with an AMD Threadripper Pro 5975WX CPU (32 cores / 64 threads), an NVIDIA RTX 4090 GPU, and 128 GB of RAM. Classical machine learning models (RF, XGBoost, LogReg, and LR-Pol), including Bayesian hyperparameter optimization, were executed on CPU using multi-core parallelism when supported by the corresponding libraries. GPU acceleration was used only for the MLP and TabPFN models.

\subsection{Numerical stability in classifier-based estimation}
\label{app:numerical-stability}

The estimator in Equation \ref{eq:js_e} involves logarithmic and ratio terms that are numerically sensitive when classifier outputs approach the boundaries of $[0,1]$. This appendix formalizes the sources of instability and describes the adjustments applied in all experiments.

\subsubsection{Sources of instability}

The estimator requires evaluating $\log \hat{D}(x)$ and $\log(1 - \hat{D}(x))$, as well as the ratio 
$\hat{D}(x)/(1-\hat{D}(x))$ used in the imbalance correction. These expressions become undefined when 
$\hat{D}(x) \to 0$ or $\hat{D}(x) \to 1$, which occurs in two distinct scenarios.

\paragraph{Disjoint or near-disjoint supports.}
Let $D^*(x) = p(x)/(p(x)+q(x))$ be the Bayes-optimal discriminator. If $x \in \mathrm{supp}(P)$ but $x \notin \mathrm{supp}(Q)$, then $q(x) = 0$ and

\[
D^*(x) = \frac{p(x)}{p(x) + 0} = 1,
\]
making $1 - D^*(x) = 0$ and the ratio $D^*/(1-D^*)$ undefined. Symmetrically, if $x \in \mathrm{supp}(Q)$ but $x \notin \mathrm{supp}(P)$, then $D^*(x) = 0$ and $\log D^*(x) \to -\infty$.

\paragraph{Finite-sample saturation in high-capacity models.}
Even when $\mathrm{supp}(P) \cap \mathrm{supp}(Q) \neq \emptyset$, high-capacity classifiers (e.g., MLPs, random forests) can overfit the decision boundary on finite samples. When the model finds a separating hyperplane that perfectly classifies training data, sigmoid activations saturate, driving $\hat{D}(x)$ to values indistinguishable from $0$ or $1$ under IEEE 754 \cite{ieee754} floating-point arithmetic. 
Specifically, if $q(x) < 10^{-15}$ relative to $p(x)$, floating-point rounding yields $\hat{D}(x) = 1.0$ exactly, making the subtraction $1 - \hat{D}(x) = 0$.

\subsubsection{Probability clipping}

To prevent undefined operations in all downstream computations, we apply a clipping transformation to all classifier outputs before evaluation:
\[
\tilde{D}(x) = \min\{1-\varepsilon,\, \max\{\varepsilon,\, 
\hat{D}(x)\}\}, \quad \varepsilon = 10^{-6}.
\]
This prevents the following failure modes:
\begin{enumerate}
    \item \textbf{Logarithm underflow:} $\hat{D}(x) \to 0$ causes $\log \hat{D}(x) \to -\infty$.
    \item \textbf{Division by zero:} $\hat{D}(x) \to 1$ causes $1 - \hat{D}(x) \to 0$, making the ratio $\hat{D}/(1-\hat{D})$ undefined.
\end{enumerate}
The value $\varepsilon = 10^{-6}$ is chosen small enough to affect only boundary values without materially altering the estimator on well-separated interior points.

\subsubsection{Negative divergence estimates}

Due to the nature of the estimator, finite-sample effects and limited model capacity may yield $\widehat{\mathrm{D}}_{JS}(P \| Q) < 0$. This does not contradict the theoretical non-negativity of $\DJS$, but reflects estimator variance when the classifier fails to distinguish between the two distributions reliably. In practice, negative estimates indicate that the classifier performs at or below chance level, and should be interpreted as $\widehat{\mathrm{D}}_{JS} \approx 0$ rather than as a meaningful divergence value.


\subsection{Additional experiments and analysis}
\label{app:additional_experiments}

These experiments extend the analysis to settings involving model misspecification and generative approximation. We complement the main experiments with additional controlled settings to analyze estimator behavior across different distributional families and data-generation models. In all cases, we consider scenarios in which the underlying distributions are known, allowing us to isolate the effects of sample size and model mismatch on divergence estimation.

Unless otherwise stated, divergence is estimated using the classifier-based pipeline described in Section \ref{sec:methodology}, and results are averaged across multiple random seeds to assess variability.

\subsubsection{Finite samples}
We evaluate the effect of finite sample size on divergence estimation. The reference distribution $P$ is a correlated bimodal Gaussian mixture in $d=2$, with two components, weights $[0.7,0.3]$, means $[(1,1),(-1,-1)]$, and within-component correlation $\rho=0.2$. We vary the dataset size $N \in \{10,\dots,100\}$ while fixing the training and evaluation budgets to the largest available $M$ and $L$, so that changes in error reflect the effect of finite-sample availability rather than changes in the evaluation protocol.
 
\begin{table}[ht]
\centering
\scriptsize
\setlength{\tabcolsep}{2.2pt}
\caption{GMM approximation experiment: absolute estimation error with respect to the Monte Carlo reference, $100\cdot|\widehat{D}_{\mathrm{JS}} - D_{\mathrm{JS}}^{\star}|$, for $M = L = 2000$. Values are reported as mean $\pm$ std over seeds.}
\label{tab:uc3_abs_error_vs_gt_mean_std_x100}
\resizebox{\textwidth}{!}{%
\begin{tabular}{rcccccccc}
\toprule
N & MLP & RF & XGBoost & LogReg & LR-Pol & TabPFN & Syndat & Synthcity \\
\midrule
10 & 6.66 $\pm$ 12.15 & 11.10 $\pm$ 14.26 & 18.17 $\pm$ 5.68 & 16.02 $\pm$ 7.75 & 23.56 $\pm$ 1.36 & 5.37 $\pm$ 13.15 & 16.59 $\pm$ 4.26 & 17.35 $\pm$ 3.97 \\
20 & 1.13 $\pm$ 8.38 & 8.80 $\pm$ 3.42 & 9.70 $\pm$ 3.72 & 9.43 $\pm$ 3.88 & 9.60 $\pm$ 3.63 & 7.44 $\pm$ 8.37 & 8.84 $\pm$ 2.55 & 9.08 $\pm$ 2.36 \\
30 & 0.01 $\pm$ 0.91 & 4.22 $\pm$ 1.74 & 4.31 $\pm$ 0.80 & 4.75 $\pm$ 1.50 & 4.48 $\pm$ 0.78 & 1.99 $\pm$ 1.00 & 3.52 $\pm$ 0.64 & 3.73 $\pm$ 0.53 \\
40 & 2.93 $\pm$ 3.70 & 6.73 $\pm$ 4.11 & 6.34 $\pm$ 2.07 & 6.77 $\pm$ 2.02 & 6.76 $\pm$ 1.35 & 4.59 $\pm$ 2.85 & 5.34 $\pm$ 1.15 & 5.52 $\pm$ 0.97 \\
50 & 1.35 $\pm$ 2.95 & 3.06 $\pm$ 3.20 & 2.97 $\pm$ 1.93 & 3.22 $\pm$ 2.96 & 3.10 $\pm$ 2.81 & 1.67 $\pm$ 4.32 & 2.35 $\pm$ 1.03 & 2.53 $\pm$ 0.96 \\
60 & 2.06 $\pm$ 2.27 & 4.79 $\pm$ 1.55 & 4.27 $\pm$ 1.15 & 4.66 $\pm$ 1.41 & 4.48 $\pm$ 0.94 & 1.68 $\pm$ 2.70 & 3.79 $\pm$ 0.69 & 4.06 $\pm$ 0.65 \\
70 & 0.56 $\pm$ 1.61 & 2.18 $\pm$ 1.19 & 2.37 $\pm$ 0.94 & 2.16 $\pm$ 0.84 & 2.45 $\pm$ 0.83 & 1.13 $\pm$ 4.09 & 1.35 $\pm$ 0.60 & 1.51 $\pm$ 0.56 \\
80 & 0.41 $\pm$ 2.15 & 2.64 $\pm$ 1.88 & 2.82 $\pm$ 1.11 & 2.65 $\pm$ 1.29 & 2.81 $\pm$ 1.40 & 1.43 $\pm$ 2.57 & 1.67 $\pm$ 0.50 & 1.85 $\pm$ 0.47 \\
90 & 0.06 $\pm$ 2.37 & 1.94 $\pm$ 1.46 & 2.04 $\pm$ 1.23 & 2.00 $\pm$ 1.14 & 2.11 $\pm$ 0.93 & 0.69 $\pm$ 3.14 & 1.29 $\pm$ 0.54 & 1.47 $\pm$ 0.54 \\
100 & 1.24 $\pm$ 1.10 & 5.02 $\pm$ 0.08 & 1.23 $\pm$ 0.56 & 1.71 $\pm$ 0.00 & 0.42 $\pm$ 0.01 & 12.58 $\pm$ 2.66 & 2.03 $\pm$ 0.34 & 1.61 $\pm$ 0.32 \\
\midrule
Time (s) & 168.36 & 174.11 & 5.25 & 31.87 & 32.56 & 0.49 & 0.00 & 0.00 \\
\bottomrule
\end{tabular}
}
\end{table}

The synthetic distribution $Q$ is obtained by fitting a Gaussian mixture model to samples drawn from $P$, and subsequently sampling from the fitted model. Unlike previous experiments, $Q$ is not fixed analytically but depends on the amount of data used for model fitting.

We vary the number of samples used to fit the Gaussian Mixture Model (GMM), $N \in \{10, \dots, 100\}$, while keeping the estimation sample sizes fixed to $M=L=2000$. This setup allows us to study how improvements in generative model quality translate into changes in the estimated divergence.

\subsubsection{Marginal-based vs joint distribution.}
\label{app:marginal_table}
Table \ref{tab:uc5_abs_error_no_gt_mean_std_x100_2dec} confirms this quantitatively: at $\rho=0.9$, both baselines incur errors of approximately $0.29$, comparable to the MC reference estimation itself and thus effectively uninformative. Among classifier-based methods, MLP, LR-Pol, RF, and TabPFN achieve errors below $0.01$ across all $\rho$ values. LogReg is a notable exception: its linear decision boundary is insufficient to capture nonlinear dependencies, leading to errors comparable to those of marginal methods at high $\rho$. This highlights that estimator capacity is critical; not all classifiers are equally suited for detecting joint-distribution discrepancies.

\begin{table}[ht]
\centering
\scriptsize
\setlength{\tabcolsep}{2.5pt}
\caption{Absolute $\DJS$ estimation error reported as mean $\pm$ std over 5 seeds. All divergence values are multiplied by 100 and shown with two decimals. Reported runtimes correspond to a single representative run and include the full training pipeline (model fitting and hyperparameter search when enabled).}
\label{tab:uc5_abs_error_no_gt_mean_std_x100_2dec}
\begin{tabular}{ccccccccc}
\toprule
$\rho$ & MLP & RF & LogReg & XGBoost & LR-Pol & TabPFN & Syndat & Synthcity \\
\midrule
0.0 & 0.00 $\pm$ 0.00 & 0.00 $\pm$ 0.01 & 0.00 $\pm$ 0.00 & 0.03 $\pm$ 0.04 & 0.00 $\pm$ 0.00 & 0.01 $\pm$ 0.01 & 0.82 $\pm$ 0.07 & 0.63 $\pm$ 0.09 \\
0.5 & 0.36 $\pm$ 0.61 & 0.66 $\pm$ 0.54 & 4.54 $\pm$ 0.19 & 0.59 $\pm$ 0.77 & 0.36 $\pm$ 0.39 & 0.56 $\pm$ 0.49 & 3.86 $\pm$ 0.09 & 4.05 $\pm$ 0.11 \\
0.7 & 0.40 $\pm$ 0.87 & 0.65 $\pm$ 0.97 & 11.44 $\pm$ 0.29 & 0.73 $\pm$ 0.81 & 0.86 $\pm$ 0.75 & 0.76 $\pm$ 0.79 & 10.89 $\pm$ 0.17 & 11.08 $\pm$ 0.15 \\
0.9 & 0.50 $\pm$ 0.55 & 1.11 $\pm$ 1.04 & 29.49 $\pm$ 0.36 & 1.34 $\pm$ 1.19 & 0.86 $\pm$ 1.10 & 0.73 $\pm$ 1.07 & 29.10 $\pm$ 0.12 & 29.27 $\pm$ 0.11 \\
\midrule
Time (s) &  94.02 & 181.76  & 0.31 & 73.46 & 31.73  & 0.45 & 0.00  & 0.00  \\
\bottomrule
\end{tabular}
\end{table}

\subsubsection{Class imbalance setup}  
\label{app:class_imbalance_setup}
To study the effect of prior shift, we construct datasets with controlled imbalance between samples from $P$ and $Q$. Figure \ref{fig:uc6_imbalance} illustrates the experimental setup. The number of samples from $P$ is kept fixed, while the number of samples from $Q$ is varied to achieve different imbalance ratios $N = |Q|/|P|$. This allows us to isolate the effect of class imbalance on the behavior of classifier-based $\DJS$ estimators.

\begin{figure}[ht]
\centering
\includegraphics[width=0.9\linewidth]{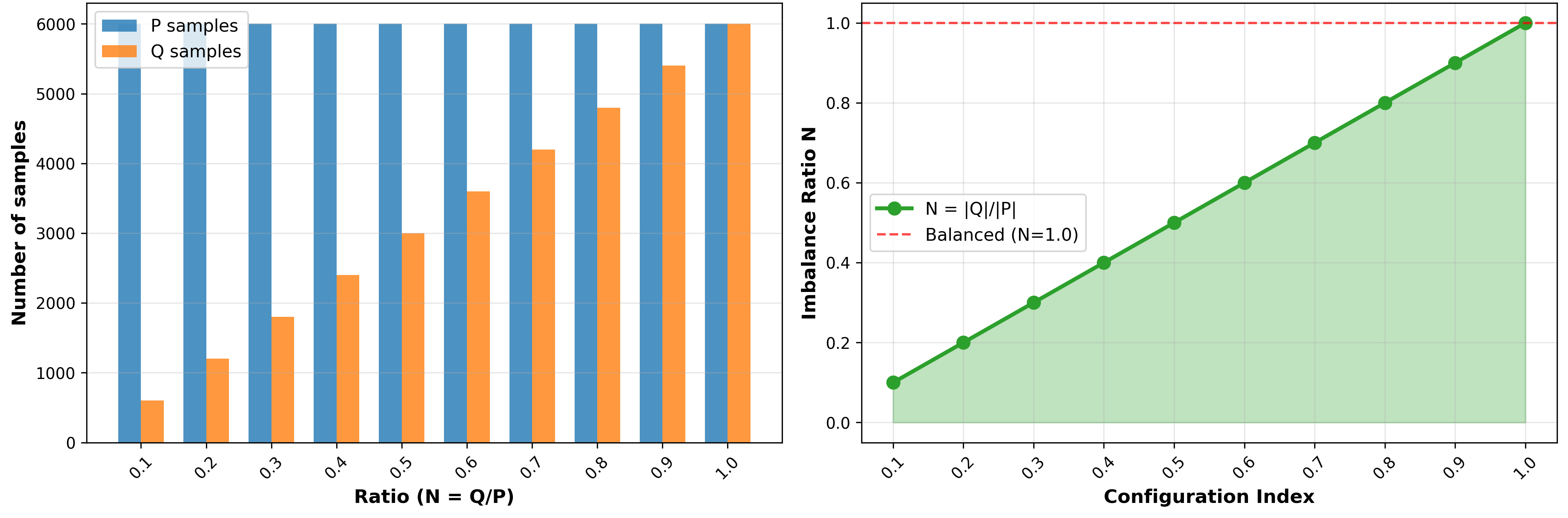}
\caption{Class imbalance setup with varying ratios $N = |Q|/|P|$, where $N=1$ denotes the balanced case.}
\label{fig:uc6_imbalance}
\end{figure}

Specifically, we set $\mu_P=(0,0)$ and shift $Q$ along the direction $(1,-1)$ as
$
\mu_Q = \mu_P + \mathrm{gap}\,(1,-1),
$
with $\mathrm{gap}\in\{0.3,0.7,1.0\}$. Figure \ref{fig:uc6_gap_scenarios} visualizes the resulting distribution gap scenarios.

\begin{figure}[ht]
\centering
\includegraphics[width=0.9\linewidth]{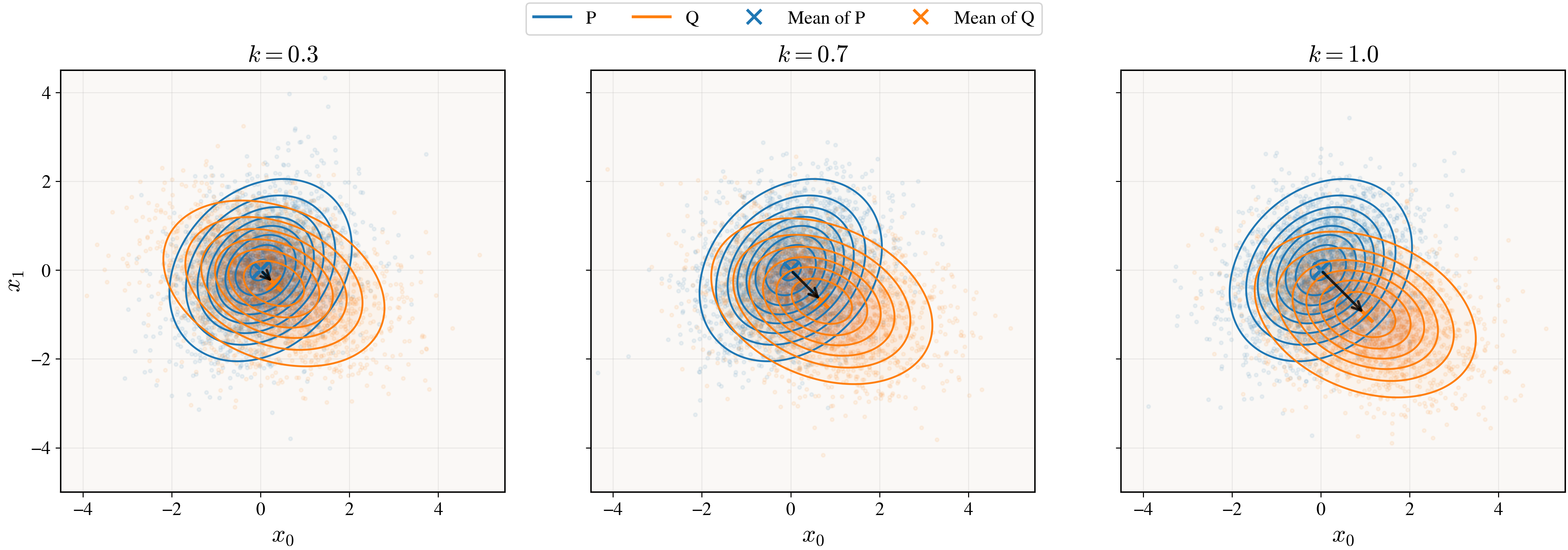}
\caption{Gaussian gap scenarios used in the class imbalance experiment. The distribution $P$ is kept fixed with mean $\mu_P=(0,0)$, while $Q$ is shifted along the direction $(1,-1)$ according to $\mu_Q=\mu_P+\mathrm{gap}(1,-1)$ for $\mathrm{gap}\in\{0.3,0.7,1.0\}$. Larger gap values correspond to more distinguishable distributions.}
\label{fig:uc6_gap_scenarios}
\end{figure}

\subsubsection{Class imbalance results}
\label{app:results_imbalance}
 
Table \ref{tab:uc13_on_abs_error_x100_all_gaps_mean_std} reports absolute estimation errors with prior correction enabled (ON), expressed as $100\cdot|\widehat{D}_{\mathrm{JS}}-D_{\mathrm{JS}}^\star|$ and averaged
over seeds, for gaps $\in\{0.3,0.7,1.0\}$ and $N\in\{0.1,0.5\}$. At $N=0.1$, MLP and TabPFN consistently achieve the lowest errors across all gap values, while RF and XGBoost show larger variance. Marginal estimators
(Syndat, Synthcity) are unaffected by the correction, as their estimates derive from feature-wise divergences rather than classifier posteriors; their errors grow with the gap, reflecting increasing divergence between
the marginal distributions. As $N\to 1$, errors decrease across all classifier-based methods, confirming that the correction is most critical under severe imbalance. Runtime is dominated by MLP and RF, while TabPFN
remains below one second regardless of gap or imbalance ratio.

\begin{table}[ht]
\centering
\scriptsize
\setlength{\tabcolsep}{2.3pt}
\caption{Absolute Jensen-Shannon estimation error multiplied by 100 with prior correction, reported as mean $\pm$ std over seeds. Values are shown as $100 \cdot |\widehat{D}_{\mathrm{JS}} - D_{\mathrm{JS}}^{\star}|$; runtime rows report mean runtime in seconds.}
\label{tab:uc13_on_abs_error_x100_all_gaps_mean_std}
\begin{tabular}{llcccccccc}
\toprule
Gap & N & MLP & RF & LogReg & LR-Pol & XGBoost & TabPFN & Syndat & Synthcity \\
\midrule
0.3 & 0.1 & 2.80 $\pm$ 2.52 & 6.19 $\pm$ 4.09 & 5.23 $\pm$ 2.40 & 5.13 $\pm$ 3.97 & 5.21 $\pm$ 3.66 & 3.59 $\pm$ 3.22 & 3.18 $\pm$ 1.19 & 3.99 $\pm$ 0.82 \\
0.3 & 0.5 & 0.84 $\pm$ 0.32 & 2.47 $\pm$ 1.75 & 4.73 $\pm$ 2.09 & 1.78 $\pm$ 1.41 & 2.07 $\pm$ 2.09 & 0.87 $\pm$ 0.20 & 5.85 $\pm$ 0.32 & 6.25 $\pm$ 0.25 \\
\midrule
0.3 & Time (s) & 332.98 & 154.52 & 27.49 & 39.49  & 72.42 & 0.34  & 0.01 & 0.01 \\
\midrule
0.7 & 0.1 & 3.62 $\pm$ 3.77 & 11.03 $\pm$ 6.08 & 6.41 $\pm$ 7.29 & 5.23 $\pm$ 3.58 & 8.91 $\pm$ 6.84 & 4.89 $\pm$ 5.04 & 7.71 $\pm$ 1.80 & 9.55 $\pm$ 1.40 \\
0.7 & 0.5 & 1.53 $\pm$ 1.19 & 2.34 $\pm$ 2.18 & 4.39 $\pm$ 3.48 & 2.16 $\pm$ 2.45 & 2.95 $\pm$ 3.61 & 1.80 $\pm$ 1.15 & 9.89 $\pm$ 0.33 & 10.83 $\pm$ 0.28 \\
\midrule
0.7 & Time (s) & 251.21  & 137.89  & 28.02 & 45.81 & 71.68  & 0.34  & 0.01 & 0.01 \\
\midrule
1.0 & 0.1 & 4.14 $\pm$ 3.55 & 6.22 $\pm$ 3.32 & 6.76 $\pm$ 8.11 & 6.84 $\pm$ 4.93 & 12.22 $\pm$ 9.97 & 5.58 $\pm$ 5.39 & 12.12 $\pm$ 2.74 & 15.40 $\pm$ 2.06 \\
1.0 & 0.5 & 2.04 $\pm$ 1.69 & 2.30 $\pm$ 2.28 & 4.50 $\pm$ 4.32 & 2.58 $\pm$ 2.42 & 2.27 $\pm$ 2.22 & 2.29 $\pm$ 1.66 & 13.79 $\pm$ 0.42 & 15.33 $\pm$ 0.39 \\
\midrule
1.0 & Time (s) & 324.47 & 144.34 & 28.22 & 47.85 & 66.95 & 0.34 & 0.00 & 0.01  \\
\bottomrule
\end{tabular}
\end{table}

\subsubsection{High Dimensionality}
\label{app:results_dimension}

Table \ref{tab:uc7_abs_error_no_gt_mean_std_x100_2dec_time_range} reports the full numerical results for the high-dimensional experiment in Section \ref{subsec:results_highdim}. Absolute estimation error grows monotonically with $d$ for all methods, but the gap between classifier-based and marginal estimators widens substantially: at $d=50$,
Syndat and Synthcity incur errors roughly five times larger than the best classifier-based methods. Among classifier-based estimators, RF shows the largest degradation, while MLP, LogReg, and TabPFN remain competitive
across all dimensions. Runtime ranges reflect the combinatorial cost of LR-Pol, which exceeds 29 minutes at $d=50$, and the near-zero cost of marginal methods.

\begin{table}[ht]
\centering
\scriptsize
\setlength{\tabcolsep}{2.5pt}
\caption{Absolute Jensen-Shannon estimation error across dimensionalities, reported as mean $\pm$ std over 5 seeds. All divergence values are multiplied by 100 and shown with two decimals. The last row reports runtime as $[\min, \max]$ over the smallest and largest tested dimensions. Reported times correspond to a single representative run and include the full training pipeline, including model fitting and hyperparameter search when enabled.} 
\label{tab:uc7_abs_error_no_gt_mean_std_x100_2dec_time_range}
\resizebox{\textwidth}{!}{%
\begin{tabular}{ccccccccc}
\toprule
d & MLP & RF & LogReg & XGBoost & LR-Pol & TabPFN & Syndat & Synthcity \\
\midrule
\midrule
2 & 0.52 $\pm$ 0.35 & 0.82 $\pm$ 5.73 & 0.28 $\pm$ 5.74 & 0.83 $\pm$ 5.71 & 0.34 $\pm$ 5.77 & 0.16 $\pm$ 5.74 & 2.70 $\pm$ 4.84 & 3.22 $\pm$ 4.76 \\
10 & 1.38 $\pm$ 0.89 & 2.33 $\pm$ 7.08 & 0.47 $\pm$ 7.35 & 1.68 $\pm$ 7.04 & 0.48 $\pm$ 7.34 & 0.33 $\pm$ 7.35 & 23.75 $\pm$ 4.39 & 24.28 $\pm$ 4.38 \\
25 & 3.16 $\pm$ 0.61 & 9.73 $\pm$ 6.76 & 3.77 $\pm$ 7.41 & 6.25 $\pm$ 7.13 & 3.70 $\pm$ 7.34 & 3.37 $\pm$ 7.51 & 39.46 $\pm$ 5.41 & 39.99 $\pm$ 5.41 \\
40 & 7.19 $\pm$ 1.63 & 18.28 $\pm$ 5.78 & 7.94 $\pm$ 6.60 & 10.19 $\pm$ 6.54 & 7.94 $\pm$ 6.60 & 7.64 $\pm$ 6.53 & 53.02 $\pm$ 4.10 & 53.54 $\pm$ 4.07 \\
50 & 7.62 $\pm$ 2.44 & 23.61 $\pm$ 6.02 & 8.84 $\pm$ 7.58 & 11.97 $\pm$ 7.06 & 8.85 $\pm$ 7.58 & 8.62 $\pm$ 7.55 & 60.03 $\pm$ 4.34 & 60.56 $\pm$ 4.32 \\
\midrule
Time [min,max] (s) &  [64.99, 69.91] & [160.28, 545.74] & [0.44, 1.47] & [71.52, 317.88] & [32.17, 1754.08] & [0.55, 1.08] & [0.01, 0.09] & [0.01, 0.06] \\
\bottomrule
\end{tabular}}
\end{table}

\subsubsection{Real-world settings}
Tables \ref{tab:uc4_big_data_adult_intrusion_vae_ctgan} and \ref{tab:uc14_low_data_adult_intrusion_vae_ctgan} report $\widehat{D}_{\mathrm{JS}}$ estimates and runtimes for the high-sample and low-sample regimes, respectively, across all dataset-generator pairs. Estimates vary by orders of magnitude across estimator families,
consistent with the protocol-dependence discussed in Section \ref{subsec:real}. Notably, LogReg and marginal methods (Syndat, Synthcity) systematically report near-zero divergence even in cases where high-capacity estimators (XGBoost, TabPFN) detect large discrepancies, suggesting that low-capacity models underestimate nonlinear multivariate structure. Runtime differences are also substantial: LR-Pol requires up to seven hours on the \textit{Intrusion} dataset in the high-sample regime, while TabPFN completes estimation in
under two seconds.

\label{app:results_real}
\begin{table}[ht]
\centering
\scriptsize
\setlength{\tabcolsep}{2.2pt}
\caption{Big-data setting for adult and intrusion with VAE and CTGAN. Rows report mean $\pm$ standard deviation over seeds for $\widehat{D}_{\mathrm{JS}}$ and runtime.}
\label{tab:uc4_big_data_adult_intrusion_vae_ctgan}
\resizebox{\textwidth}{!}{%
\begin{tabular}{llrrrrrrrr}
\toprule
Generator & Dataset & MLP & RF & LogReg & XGBoost & LR-Pol & TabPFN & Syndat & Synthcity \\
\midrule
\multicolumn{10}{c}{\textbf{$\widehat{D}_{\mathrm{JS}}$}} \\
\midrule
VAE & adult & 0.19 $\pm$ 0.01 & 0.21 $\pm$ 0.01 & 0.03 $\pm$ 0.00 & 0.69 $\pm$ 0.03 & 0.05 $\pm$ 0.00 & 0.41 $\pm$ 0.01 & 0.01 $\pm$ 0.00 & 0.01 $\pm$ 0.00 \\
CTGAN & adult & 0.12 $\pm$ 0.01 & 0.17 $\pm$ 0.02 & 0.02 $\pm$ 0.01 & 0.67 $\pm$ 0.01 & 0.04 $\pm$ 0.01 & 0.38 $\pm$ 0.02 & 0.01 $\pm$ 0.00 & 0.01 $\pm$ 0.00 \\
VAE & intrusion & 0.92 $\pm$ 0.01 & 0.99 $\pm$ 0.01 & 0.05 $\pm$ 0.01 & 1.00 $\pm$ 0.00 & 0.22 $\pm$ 0.01 & 0.99 $\pm$ 0.00 & 0.04 $\pm$ 0.00 & 0.02 $\pm$ 0.00 \\
CTGAN & intrusion & 0.56 $\pm$ 0.03 & 0.84 $\pm$ 0.01 & 0.07 $\pm$ 0.01 & 0.90 $\pm$ 0.01 & 0.29 $\pm$ 0.01 & 0.83 $\pm$ 0.01 & 0.02 $\pm$ 0.00 & 0.01 $\pm$ 0.00 \\
\midrule
\multicolumn{10}{c}{\textbf{Time (s)}} \\
\midrule
VAE & adult & 932.60 & 425.61 & 54.80 & 158.24 & 3733.98 & 1.55  & 0.04 & 0.04  \\
CTGAN & adult & 955.84  & 402.26  & 50.42 & 151.99 & 3423.97 & 1.45 & 0.05  & 0.04 \\
VAE & intrusion & 1122.72 & 311.83  & 864.81  & 136.01  & 25275.99  & 1.73 & 0.18 & 0.12 \\
CTGAN & intrusion & 1163.12  & 296.91  & 734.26  & 164.56  & 28076.65 & 1.76 & 0.08  & 0.07  \\
\bottomrule
\end{tabular}%
}
\end{table}

\begin{table}[ht]
\centering
\scriptsize
\setlength{\tabcolsep}{2.2pt}
\caption{Low-data setting (10\%) for adult and intrusion with VAE and CTGAN. Rows report mean $\pm$ standard deviation over seeds for $\widehat{D}_{\mathrm{JS}}$ and runtime.}
\label{tab:uc14_low_data_adult_intrusion_vae_ctgan}
\resizebox{\textwidth}{!}{%
\begin{tabular}{llrrrrrrrr}
\toprule
Generator & Dataset & MLP & RF & LogReg & XGBoost & LR-Pol & TabPFN & Syndat & Synthcity \\
\midrule
\multicolumn{10}{c}{\textbf{$\widehat{D}_{\mathrm{JS}}$}} \\
\midrule
VAE & adult & 0.04 $\pm$ 0.03 & 0.09 $\pm$ 0.04 & 0.02 $\pm$ 0.02 & 0.46 $\pm$ 0.03 & 0.02 $\pm$ 0.02 & 0.29 $\pm$ 0.03 & 0.01 $\pm$ 0.00 & 0.01 $\pm$ 0.00 \\
CTGAN & adult & 0.03 $\pm$ 0.01 & 0.11 $\pm$ 0.03 & 0.03 $\pm$ 0.01 & 0.42 $\pm$ 0.07 & 0.02 $\pm$ 0.01 & 0.24 $\pm$ 0.02 & 0.02 $\pm$ 0.00 & 0.01 $\pm$ 0.00 \\
CTGAN & intrusion & 0.24 $\pm$ 0.10 & 0.73 $\pm$ 0.06 & 0.03 $\pm$ 0.02 & 0.86 $\pm$ 0.04 & 0.07 $\pm$ 0.04 & 0.81 $\pm$ 0.03 & 0.02 $\pm$ 0.00 & 0.01 $\pm$ 0.00 \\
VAE & intrusion & 0.70 $\pm$ 0.12 & 0.95 $\pm$ 0.02 & 0.00 $\pm$ 0.00 & 0.96 $\pm$ 0.02 & 0.00 $\pm$ 0.01 & 0.99 $\pm$ 0.00 & 0.04 $\pm$ 0.00 & 0.01 $\pm$ 0.00 \\

\midrule
\multicolumn{10}{c}{\textbf{Time (s)}} \\
\midrule
VAE & adult & 29.04 $\pm$ 5.97 & 119.54 $\pm$ 8.54 & 40.21 $\pm$ 1.21 & 72.82 $\pm$ 4.59 & 743.07 $\pm$ 98.31 & 0.47 $\pm$ 0.21 & 0.01 $\pm$ 0.00 & 0.01 $\pm$ 0.00 \\
CTGAN & adult & 27.83 $\pm$ 7.16 & 120.73 $\pm$ 12.25 & 39.92 $\pm$ 2.61 & 69.97 $\pm$ 4.38 & 783.67 $\pm$ 96.29 & 0.32 $\pm$ 0.07 & 0.01 $\pm$ 0.00 & 0.01 $\pm$ 0.00 \\
VAE & intrusion & 53.62 $\pm$ 22.31 & 119.61 $\pm$ 11.68 & 92.08 $\pm$ 10.16 & 63.14 $\pm$ 1.84 & 1878.24 $\pm$ 498.49 & 0.49 $\pm$ 0.16 & 0.02 $\pm$ 0.00 & 0.02 $\pm$ 0.00 \\
CTGAN & intrusion & 53.66 $\pm$ 46.11 & 125.71 $\pm$ 2.09 & 129.30 $\pm$ 28.57 & 71.08 $\pm$ 1.13 & 2452.68 $\pm$ 103.78 & 0.41 $\pm$ 0.00 & 0.02 $\pm$ 0.00 & 0.02 $\pm$ 0.00 \\
\bottomrule
\end{tabular}%
}
\end{table} 

\subsection{Software package and reproducibility}
\label{app:software}

We provide an open-source Python package implementing the full $\DJSjoint$ estimation pipeline introduced in this work. The package is designed to be modular, extensible, and easily integrated into existing synthetic data evaluation workflows.

\paragraph{Core functionality.}
The package includes a unified interface for all classifier-based estimators considered in this paper (MLP, RF, Logistic Regression, LR-Pol, XGBoost, TabPFN), with automated hyperparameter search enabled by default. It also implements the closed-form prior correction for class imbalance (Section \ref{subsec:prior_correction}), which is applied automatically when sample sizes differ. The API is designed to closely mirror \texttt{\href{https://docs.scipy.org/doc/scipy/reference/generated/scipy.spatial.distance.jensenshannon.html}{scipy.spatial.distance.jensenshannon}} \cite{scipy_jensenshannon}, enabling users to estimate $\DJSjoint$ with a single function call.

\paragraph{Benchmarking utilities.}
To support systematic evaluation, we provide utilities for computing MC reference estimation of $\DJS$ in controlled Gaussian settings, facilitating benchmarking and comparison of new estimators.

\paragraph{Reproducibility.}
We release all experimental configurations from Section \ref{sec:results}, including random seeds, dataset splits, and evaluation protocols, enabling exact reproduction of the reported results.

The full estimation pipeline is available as an open-source Python package:
\begin{itemize}
    \item \textbf{GitHub:} \url{https://github.com/AlbaGarridoLopezz/jensenshannondivergence}
    \item \textbf{PyPI:} \url{https://pypi.org/project/jensenshannondivergence/}
    \item \textbf{Installation:} \texttt{pip install jensenshannondivergence}
\end{itemize}

Our goal is for this package to serve as a community reference for reliable 
$\DJSjoint$ estimation in synthetic data evaluation, lowering the barrier to 
principled, estimator-aware reporting.


\end{document}